%% file: main.tex
\documentclass[conference]{IEEEtran}
\usepackage{times}
\usepackage{graphicx, float, url, wrapfig}
\usepackage{multicol}
\usepackage[bookmarks=true]{hyperref}
\usepackage{algorithm}
\usepackage{algpseudocode}

\usepackage{multicol, multirow}
\usepackage{array, booktabs}

\usepackage{xcolor}
\usepackage{bbm}
\usepackage{amsmath, amssymb}
\usepackage{mathtools}
\usepackage{amsthm}
\usepackage{thm-restate}
\usepackage[numbers, sort&compress]{natbib}
\newif\ifarxiv
\arxivtrue

\begin{document}

\input{RSS/header.sty}

\title{Failure Prediction with Statistical Guarantees \\ for Vision-Based Robot Control}

\ifarxiv
\author{
\authorblockN{Alec Farid\authorrefmark{1} \ \ \ \ David Snyder\authorrefmark{1} \ \ \ \ Allen Z. Ren \ \ \ \ Anirudha Majumdar}
\authorblockA{
Department of Mechanical and Aerospace Engineering, Princeton University
\\ \texttt{\{afarid, dasnyder, allen.ren, ani.majumdar\}@princeton.edu}
}}

\fi



%

\maketitle

\begin{abstract}
We are motivated by the problem of performing \emph{failure prediction} for safety-critical robotic systems with high-dimensional sensor observations (e.g., vision). Given access to a black-box control policy (e.g., in the form of a neural network) and a dataset of training environments, we present an approach for synthesizing a failure predictor with \emph{guaranteed bounds} on false-positive and false-negative errors. In order to achieve this, we utilize techniques from \emph{Probably Approximately Correct (PAC)-Bayes} generalization theory. In addition, we present novel \emph{class-conditional} bounds that allow us to trade-off the relative rates of false-positive vs. false-negative errors. We propose algorithms that train failure predictors (that take as input the history of sensor observations) by minimizing our theoretical error bounds. We demonstrate the resulting approach using extensive simulation and hardware experiments for vision-based navigation with a drone and grasping objects with a robotic manipulator equipped with a wrist-mounted RGB-D camera. These experiments illustrate the ability of our approach to (1) provide strong bounds on failure prediction error rates (that closely match empirical error rates), and (2) improve safety by predicting failures. \blfootnote{\authorrefmark{1} Equal Contribution}  
\end{abstract}

\IEEEpeerreviewmaketitle

\input{RSS/Sections/Intro}

\input{RSS/Sections/RelatedWork}

\input{RSS/Sections/ProblemFormulation}

\input{RSS/Sections/Approach}

\input{RSS/Sections/Exp_Results}

\input{RSS/Sections/Conclusion}

\ifarxiv
\section*{Acknowledgments}
\label{sec:Ack}
\noindent The authors were partially supported by the NSF CAREER Award [\#2044149], the Office of Naval Research [N00014- 21-1-2803,  N00014-18-1-2873], the Toyota Research Institute (TRI), the NSF Graduate Research Fellowship Program [DGE-2039656], and the School
of Engineering and Applied Science at Princeton University through the generosity of William Addy ’82. This article solely reflects the opinions and conclusions of its authors and not NSF, ONR, Princeton SEAS, TRI or any other Toyota entity.
\fi

\bibliographystyle{unsrtnat}
\bibliography{main}

\begin{appendices}
\input{RSS/Sections/Appendices/Class_Conditional_Bound}
\input{RSS/Sections/Appendices/ToyExample}
\input{RSS/Sections/Appendices/Experimental_details}
\end{appendices}
\end{document}

%% file: RSS/Sections/Intro.tex
\section{Introduction}
\label{sec:Intro}
How can we guarantee the safety of a control policy for a robot that operates using high-dimensional sensor observations (e.g., a vision-based navigation policy for a drone; Fig.~\ref{fig:anchor})? This is particularly challenging for policies that have learning-based components such as neural networks as part of the perception and control pipeline; state-of-the-art approaches for synthesizing such policies (e.g., based on deep reinforcement learning) do not provide guarantees on safety, and can lead to policies that fail catastrophically in novel environments. 

Motivated by this challenge, we consider the following problem in this paper. Given access to a black-box control policy (e.g., one with neural network components), our goal is to train a \emph{failure predictor} for this policy. This failure predictor acts as a `safety layer' that is responsible for predicting (online) if the given policy will lead to a failure as the robot operates in a novel (i.e., previously unseen) environment. Such a predictor could enable the robot to deploy a backup policy (e.g., transitioning to hover) in order to ensure safety. In certain settings (e.g., factories with robotic manipulators), the triggering of a failure predictor could also allow the robot to seek help from a human supervisor. We envision that the addition of a failure predictor could substantially improve the safety of the overall robotic system.

In order to be deployed in safety-critical settings, a failure predictor should ideally have associated \emph{formal guarantees} on its prediction performance. More precisely, we would like to upper bound the \emph{false negative} and \emph{false positive} rates of the predictor. Here, a false negative corresponds to a case where the policy leads to a failure, but the failure predictor does not predict this in advance. Similarly, a false positive is a case where the failure predictor triggers, but the policy remains safe. In addition to bounding these error rates, one would also ideally like to \emph{trade-off} the relative proportion of false negatives vs. false positives. For example, ensuring a low false negative rate is generally more crucial in safety-critical settings (potentially at the cost of a higher false positive rate). 

\begin{figure}[t]
\centering
\includegraphics[width=0.45\textwidth]{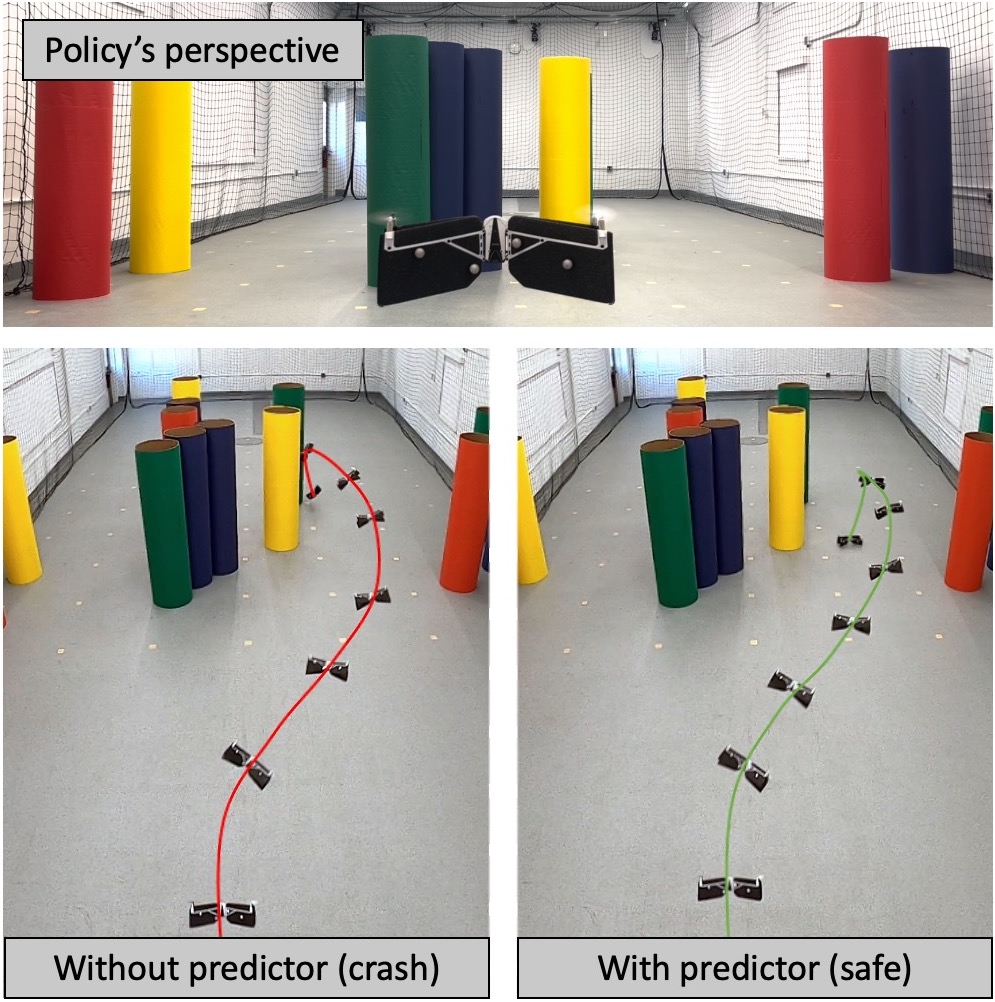}
\vspace{-5pt}
\caption{We train a failure predictor which guarantees (with high probability) detection of a failure ahead of time. 
A policy is tasked with avoiding obstacles in a novel environment using a first-person depth image (top). The bottom images show the entire environment, including multiple obstacles which are occluded from view. Since it is operating in a novel setting, the policy alone fails to avoid occluded obstacles and crashes (left). When the failure predictor is deployed, the policy stops safely before a crash (right).}
\label{fig:anchor}
\vspace{-15pt}
\end{figure}

\textbf{Statement of contributions.} Our primary contribution is to develop an approach for learning a failure predictor with guaranteed bounds on error rates, given a black-box control policy that operates on a robotic system with high-dimensional sensor observations (e.g., vision or depth). Assuming access to a dataset of environments on which we can execute the black-box policy for training purposes, 
we present a \emph{reduction} of the failure prediction problem (which involves non-i.i.d. data in the form of sensor observations) to a standard supervised learning setting based on the reduction presented in \cite{Majumdar20}. This allows us to utilize techniques from \emph{Probably Approximately Correct (PAC)-Bayes} generalization theory in order to obtain bounds on error rates. In addition, we develop novel \emph{class-conditional} bounds that allow us to trade-off the relative rates of false negative vs. false positive errors. Our algorithmic approach is then to train a failure predictor (e.g., in the form of a neural network that takes as input a history of sensor observations) by minimizing the theoretical error bounds. We demonstrate our approach on two examples in simulation and hardware including vision-based navigation with a drone and grasping objects with a robotic manipulator equipped with a wrist-mounted RGB-D camera. These experiments illustrate the ability of our approach to (1) provide strong bounds on failure prediction error rates, and (2) improve the safety of a robotic system by applying our failure prediction scheme. 


%% file: RSS/Sections/RelatedWork.tex
\section{Related Work}
\label{sec:RW}



\textbf{Anomaly Detection.} 
There is a long line of work on \emph{anomaly detection} in the signal processing literature (see \cite{Basseville88} for a review). The goal is to identify events or inputs (e.g., sensor observations) that deviate significantly from nominal inputs (and may thus cause failures). In recent years, there have been efforts to develop anomaly detection techniques for supervised learning problems with high-dimensional inputs such as images \cite{Hendrycks17, Liang18, DeVries18, Lee18, Bitterwolf20, Chen22} (see \cite{Ruff21} for a recent review). This line of work is also closely related to the literature on uncertainty estimation in deep learning \cite{Abdar21}. Recent techniques also allow for anomaly detection in reinforcement learning settings using streams of sensor inputs \cite{Sedlmeier20, Cai20, Wu21, Greenberg21}. Within the robotics literature, there has been work on performing uncertainty estimation \cite{Kahn17, Lutjens19} or anomaly detection \cite{Richter17} specifically for the problem of collision avoidance (which is an example we consider in this paper). In contrast to our goal of failure prediction, the methods highlighted above are generally aimed at detecting anomalies/outliers in environments that the robot is operating in. While anomalies may result in failures, this is not necessarily the case (e.g., for an extremely unlikely yet benign environment). In addition, the methods above do not generally provide guarantees on false positive or negative rates of detection, or the improvement in the robot's performance that results from applying the anomaly detector. 

{\bf Failure Prediction.} A different line of work foresees failure by explicitly forward-propagating the robot dynamics in the environment. For instance, if the environment map is available (either given or estimated online), and the robot dynamics are known (e.g., car), the robot can check if it falls into a failure state in the next step before executing the action. Typical methods include Hamilton-Jacobi reachability analysis~\cite{mitchell2005time, akametalu2014reachability}, control barrier functions \cite{ames2016control}, and formal methods~\cite{alshiekh2018safe}. However, these approaches typically assume an explicit description of the uncertainty affecting the system (e.g., bound on actuation noise) and/or the environment (e.g., minimum distance between obstacles), which are often unrealistic to describe with real-world environments. They also work with low-dimensional systems, and are generally unable to provide guarantees for policies trained with rich sensing like vision. Work in safe reinforcement learning \cite{dalal2018safe, thananjeyan2021recovery} approximates the outcome of robot dynamics from data and allows high-dimensional image input; however, they do not provide guarantees on the robot's safety.

Recently, techniques from \emph{conformal prediction}~\cite{Vovk_CP_book_2005, shafer_CP_tutorial_2008} have been used to perform failure prediction with error bounds~\cite{luo_CP_safety_2021}. This method provides single-episode generalization with the milder assumption of training data exchangeability, as opposed to the i.i.d. assumption on the training data for PAC learning (see e.g. \cite{chow_Chapter7_1997} for reference). 
However, because the resulting guarantee fundamentally operates on a `centroidal measure' of the performance distribution, it does not generally provide high-probability guarantees over the predictors synthesized from the training data. A modification of conformal prediction that does give high-probability guarantees, however, is computationally intractable except for very simple settings \cite{foygel_CP_limits_2021}. We provide a comparison of our approach with conformal prediction in Sec.~\ref{subsecApp:Class}. 


\textbf{PAC-Bayes Learning Theory.}
Generalization bounds based on PAC-Bayes theory \cite{McAllester99, Shawe-Taylor97, Seeger02} have recently been shown to provide strong guarantees for a variety of large-scale supervised learning settings \cite{Langford03, Germain09, Dziugiate17, Bartlett17, Jiang20, Perez-Ortiz20, Viallard21}. Since their original development, there has been significant work on strengthening and extending PAC-Bayes bounds \cite{Catoni04, Catoni07, McAllester13, Rivasplata19, Thiemann17, Dziugaite18}. Previous work has also extended PAC-Bayes theory in order to learn control policies for robots with guarantees on generalization to novel environments \cite{Majumdar20, Veer20, Ren20}. This work requires implementing a specific training pipeline for learning a policy (by optimizing a PAC-Bayes generalization bound). In contrast, we seek to perform failure prediction for \emph{any} given (black-box) policy. PAC-Bayes theory allows us to provide guarantees on the false positive and negative rates of failure prediction. Our method can thus be used in a `plug-and-play' manner to improve the safety of an existing policy. One recent work \cite{hsu2022sim} also improves safety of the policies trained using PAC-Bayes theory and detects failure by learning a safety value function using Hamilton-Jacobi reachability-based reinforcement learning, but it does not explicitly provide guarantees on failure prediction.

%% file: RSS/Sections/ProblemFormulation.tex
\section{Problem Formulation} 
\label{sec:Problem}

\subsection{Notation}
\label{subsecPF:Notation}
In this paper, general spaces are denoted by capitalized Greek letters and calligraphic lettering. As examples, $\Pi$ represents the space of black-box policies and $\F$ represents the space of failure prediction hypotheses. The only exception, $\D$, instead denotes probability distributions. A subscript on a distribution denotes the (assumed measurable where needed) space over which the distribution acts; e.g. $\D_\E$ denotes the distribution over the space of environments. Unless otherwise noted, samples from a distribution are always assumed to be i.i.d.. Let $\cap$ denote the joint event of a predicted class and a true class. For example, $p_{0 \cap 1}$ denotes the probability of the joint event `predict class $0$ and true class is $1$,' which is exactly the probability of false negative. \emph{Class-conditional} prediction probabilities follow the same convention. For example, $p_{0 \rvert 1}$ is the conditional probability of `predicting class 0 given the true class is 1.' The indicator function $\mathbbm{1}(x) = 1$ if $x$ is true and $\mathbbm{1}(x) = 0$ if $x$ is false.


\subsection{Problem Setting}
\label{subsecPF:PS}

\textbf{Environment and Task Policy.} We consider a setting where environments $E$ are drawn from an underlying distribution $\D_\E$ over environments. Here, `environment' refers to factors that are external to the robot (e.g., an obstacle field for drone navigation, or objects for robotic grasping). We do not assume any knowledge of or structure on $\D_\E$, and only assume indirect access to the distribution via a dataset $S = \{E_1, E_2, ..., E_N\}$ of i.i.d. samples from $\D_\E$. At test time, we must perform failure prediction in new environments drawn from $\D_\E$. In all environments, we assume the robot has a sensor that provides an observation $o \in \mathcal{O}$ (e.g., RGB-D image) at each time step. We consider a deterministic, black-box task policy $\pi: \mathcal{O} \rightarrow \mathcal{A}$ that maps (potentially a history of) observations to a control input. Deploying the policy in an environment induces a label $y \in \{0, 1\}$ indicating the robot's failure (e.g., $0$ if it reaches a target and $1$ if not).


\textbf{Failure Predictor.} Given a fixed policy, our goal is to train a failure predictor using the dataset $S$ of environments. The predictor hypothesis $f: \mathcal{O}^H \rightarrow \mathcal{Y}$ maps a fixed horizon of past observations to a predicted class $\hat{y} \in \{0, 1\}$. We consider that the predictor has two possible outputs: class $0$ for predicting success and class $1$ for predicting failure. Thus there are four possible outcomes: (1) true positive ($1 \cap 1$), predicting $1$ \emph{at least once before failure}; (2) true negative ($0 \cap 0$), \emph{never} predicting $1$ \emph{during the entire successful rollout}; (3) false positive ($1 \cap 0$), incorrectly predicting $1$ at least once during a successful rollout; and (4) false negative ($0 \cap 1$), incorrectly predicting $0$ \emph{at all steps of a failed rollout}. The rates of the four outcomes sum to 1: $p_{1 \cap 1} + p_{0 \cap 0} + p_{1 \cap 0} + p_{0 \cap 1} = 1$. The \emph{misclassification error} refers to the sum of the rate of false positive and false negative, $p_{1 \cap 0} + p_{0 \cap 1}$. In addition, we define \emph{false positive rate} (FPR) as the ratio between false positive and sum of false positive and true negative, $\text{FPR} := p_{1 \cap 0} / (p_{1 \cap 0} + p_{0 \cap 0}) = p_{1 \rvert 0}$; similarly, \emph{false negative rate} (FNR) is the ratio between false negative and sum of false negative and true positive, $\text{FNR} := p_{0 \cap 1} / (p_{0 \cap 1} + p_{1 \cap 1}) = p_{0 \rvert 1}$. 


Let $r_f: \E \times \Pi \rightarrow \mathcal{X}^{T} \times \mathcal{Y}^{T}$ denote the function that `rolls out' the system with the given policy and the predictor $f$ for $T$ steps, i.e., $r_f$ maps an environment $E$ to the trajectory of states $x_t$ (resulting from applying the policy) and failure predictions $\hat{y}_t$, for $t \in \{1,\dots,T\}$. To evaluate the performance of the failure predictor, we introduce the error of applying the predictor $f$ in an environment $E$ where a policy $\pi$ is running:
\begin{equation} \label{Eq:CostNominal}
    C(r_f(E, \pi)) := \mathbbm{1}[(\max_{t < T_\text{fail}} \hat{y}_t) \ \neq y],
\end{equation}
where $T_\text{fail}$ is the step when failure occurs; if the whole rollout is successful, $T_\text{fail} = T+1$. This \emph{misclassfication error} describes if the prediction is misclassified (false positive or false negative). As shown in Sec.~\ref{subsecApp:Toy}, it is often useful to consider the \emph{class-conditional misclassfication error} that allows us to trade-off false positives and false negatives:
\begin{equation} \label{Eq:CostClass}
    \tilde{C}(r_f(E, \pi)) := \sum_{y \in \{0,1\}} \lambda_y \mathbbm{1}[(\max_{t < T_\text{fail}} \hat{y}_t) \ \neq y],
\end{equation}
where $\lambda_y$ weighs the relative importance of false positives and false negatives (e.g., $\lambda_0$ = 0.3 and $\lambda_1$ = 0.7 considers false negatives to be more costly).

\textbf{Goal.} Our goal is to use the training environments $S$ to learn failure predictors that minimize the errors and \emph{provably generalize} to unseen environments drawn from the distribution $\mathcal{D}_\E$. We also employ a slightly more general formulation where a \emph{distribution} $\mathcal{D}_\mathcal{F}$ over predictors $f$ instead of a single predictor is used. Below we present the two optimization problems corresponding to misclassification error and class-conditional misclassification error:
\begin{align} 
    \inf_{\D_\F} \underset{E \sim \D_\E}{\EE} \underset{f \sim \D_\F}{\EE} \big[C(r_f(E, \pi)) \big], \label{Eq:CostMinNominal}\\
     \inf_{\D_\F} \underset{E \sim \D_\E}{\EE} \underset{f \sim \D_\F}{\EE} \big[\tilde{C}(r_f(E, \pi)) \big]. \label{Eq:CostMinClass}
\end{align}

We emphasize that we do not have direct access to the distribution $\D_\E$ over environments. 
We must thus train failure predictors that will \emph{generalize} beyond the finite training dataset $S$ of environments we assume access to.

%% file: RSS/Sections/Approach.tex
\section{Failure Prediction with Guaranteed Error Bounds} \label{sec:Pac-Bayes}
Our approach for learning failure predictors with guaranteed error bounds relies on a reduction to results from the PAC-Bayes generalization theory from supervised learning; in order to present this reduction, we first provide a brief background of PAC-Bayes in the standard supervised learning setting \cite{Dziugiate17}.

\input{RSS/Sections/Approach_Folder/NominalPAC}

\input{RSS/Sections/Approach_Folder/Toy}

\input{RSS/Sections/Approach_Folder/ClassCondPAC}

\input{RSS/Sections/Approach_Folder/Implementation}

%% file: RSS/Sections/Approach_Folder/NominalPAC.tex
\subsection{PAC-Bayes Theory for Supervised Learning}
\label{subsecApp:Background}
Consider $\mathcal{Z}$ the input space, and $\D_\mathcal{Z}$ the (unknown) true distribution on $\mathcal{Z}$. We assume access to $S_\mathcal{Z} = \{z_1, z_2, ..., z_N\}$, $N$ i.i.d. samples of data $z_i \in \mathcal{Z}$ from $\D_\mathcal{Z}$. To each $z_i$ there corresponds a class label $y_i \in \{0, 1\}$. Let $\mathcal{H}$ be a class of hypotheses consisting of functions $h: \mathcal{Z} \rightarrow \{0, 1\}$, parameterized by some $w \in \mathcal{W} \subseteq \mathbb{R}^d$ (e.g., neural networks parameterized by weights $w$). Consider a loss function $l: \mathcal{H} \times \mathcal{Z} \rightarrow [0, 1]$. Now let $\mathcal{W}$ denote the space of probability distributions on the parameter space $\mathbb{R}^d$. We assume there is a \emph{`prior' distribution} $\D_{\mathcal{W},0} \in \mathcal{W}$ before observing the $N$ samples; and afterwards, we choose a \emph{posterior distribution} $\D_\mathcal{W} \in \mathcal{W}$. Denote the training loss of the posterior $\D_\mathcal{W}$ as:
\begin{equation}
    l_S(\D_\mathcal{W}) := \frac{1}{N} \sum_{z \in S} \underset{w \sim \D_\mathcal{W}}{\EE} \big[l(w;z) \big].
\end{equation}
The following theorem allows us to bound the true expected loss (over unknown data from $D_\mathcal{Z}$) achieved by any $\D_\mathcal{W}$.
\begin{theorem}[Supervised Learning PAC-Bayes Bound \cite{McAllester99}]

\noindent For any $\delta \in (0, 1)$ and `prior' distribution $\D_{\mathcal{W},0}$ over $w$, the following inequality holds with probability at least $1-\delta$ over training sets $S \sim \mathcal{D}_\mathcal{Z}^N$ for all posterior distributions $\D_\mathcal{W}$:
\begin{equation}
    \begin{split}
        \underset{z \sim \D_\mathcal{Z}}{\EE} \underset{w \sim \D_\mathcal{W}}{\EE} & \big[l(w; z) \big] \leq  l_S(\D_\mathcal{W}) + R(\D_\mathcal{W}, \D_{\mathcal{W},0}, \delta), \\
        R(\D_\mathcal{W}, \D_{\mathcal{W},0}, \delta) & = \sqrt{\frac{\KL(\D_\mathcal{W} \rVert \D_{\mathcal{W},0}) + \log(\frac{2\sqrt{N}}{\delta})}{2N}} \ ,
    \end{split}
\end{equation}
where $\KL(\cdot \| \cdot)$ is the Kullback-Leibler (KL) divergence.  
\end{theorem}
\noindent Note that the prior distribution $\D_{\mathcal{W},0}$ may be chosen arbitrarily, but random or poorly chosen priors may negatively affect the guarantee. We elaborate on how we select the prior (and posterior) in Sec. \ref{sec:AlgorithmicApproach}.

\subsection{Bound on Misclassification Error}
\label{subsecApp:Nominal}

We now present a result that tackles \eqref{Eq:CostMinNominal} for learning failure predictors with bounds on the expected misclassification error. Assume there is a prior distribution $D_{\mathcal{F},0}$ on the failure predictors, and then denote the training error of the posterior $D_\mathcal{F}$ on the training dataset $S$ of environments as:
\begin{equation}
    C_S(\D_\mathcal{F}) := \frac{1}{N} \sum_{E \in S} \underset{f \sim \D_\mathcal{F}}{\EE} \big[C(r_f(E, \pi)) \big].
\end{equation}
Analogous to results in Sec.\ref{subsecApp:Background}, we present the following theorem that upper bounds the true expected misclassification error.
\begin{theorem}
[PAC-Bayes Bound on Misclassification Error]

\noindent For any $\delta \in (0, 1)$ and any prior distribution over failure predictors $\D_{\F,0}$, the following inequality holds with probability at least $1-\delta$ over training sets $S \sim \D_\E^N$ for all posterior distributions $\D_\F$:
\begin{equation}
\label{eq:pacbayes_bound}
    \begin{split}
        \underset{E \sim \D_\E}{\EE}\underset{f \sim \D_\F}{\EE}  \big[C(r_f(E, \pi)) \big] \leq C_S(\D_\mathcal{F}) + R(\D_\F, \D_{\F,0}, \delta) , \\
        R(\D_\F, \D_{\F,0}, \delta) = \sqrt{\frac{\KL(\D_\F \rVert \D_{\F,0}) + \log(\frac{2\sqrt{N}}{\delta})}{2N}}.
    \end{split}
\end{equation}
\end{theorem}
\begin{proof}
We apply a reduction of failure prediction to the supervised learning setting, from which the results follow immediately. An outline of the reduction is presented in Table \ref{Table:NominalProof}. See \cite{Majumdar20} for a similar proof construction. 

\begin{table}[ht]
\footnotesize
    \centering
    \caption{Reduction of Failure Prediction to Supervised Learning in PAC-Bayes Theory \label{Table:NominalProof}}
    \begin{tabular}{ccc}
        \toprule
        Supervised Learning & $\leftarrow$ & Failure Prediction \\
        \midrule
        Input Data $z \in \Z$ & 
          & Environment $E \in \E$  \\
        Hypothesis $h_w: \mathcal{Z} \rightarrow \Z'$ & & Rollout $r_f: \E \times \Pi \rightarrow \X^T \times \mathcal{Y}^T$\\
        Loss $l(w; z)$ & & Error $C(r_f(E, \pi))$ \\
        \bottomrule
    \end{tabular}
    \label{tab:reduction}
\end{table}

\textbf{Remark.} It is worth noting that the use of trajectories does not violate the i.i.d. data assumption required for PAC-Bayes theory. While, for example, the states visited along the multi-step trajectory are not i.i.d., the error and guarantee are defined for each \emph{rollout}, and thus \emph{are i.i.d.} Our guarantee holds over the course of the entire rollout; each trajectory `collapses' into a single `datum,' and the independence of these `data' (each environment/rollout) is preserved.
\end{proof}

In training, we try to minimize the right-hand side of the PAC-Bayes bound in order to find a posterior distribution $\D_\F$ that provides a tight bound on the expected error $\EE_{f \sim \D_\F} C(r_f(E, \pi))$. This corresponds to the probability that the predictor will be incorrect when deployed in a new environment from the unknown distribution $\D_\E$. Thus, on a new environment, the resulting $\D_\F$ will only make an unsuccessful prediction with probability less than the upper bound; a small upper bound guarantees a small probability of incorrect prediction. 

While this approach provides a clean upper bound on the misclassification error of the failure predictor, for certain practical cases, this can be insufficient. In the next subsection we construct a simple example to illustrate that misclassification error can lead to poor prediction. The natural solution to this problem is to use a weighted (class-conditional) misclassification error and bound FNR and FPR instead (Sec.~\ref{subsecApp:Class}).

%% file: RSS/Sections/Approach_Folder/Toy.tex
\subsection{Motivating Class-Conditional Misclassification Error}
\label{subsecApp:Toy}

\begin{figure}[t]
\centering
\includegraphics[width=0.40\textwidth]{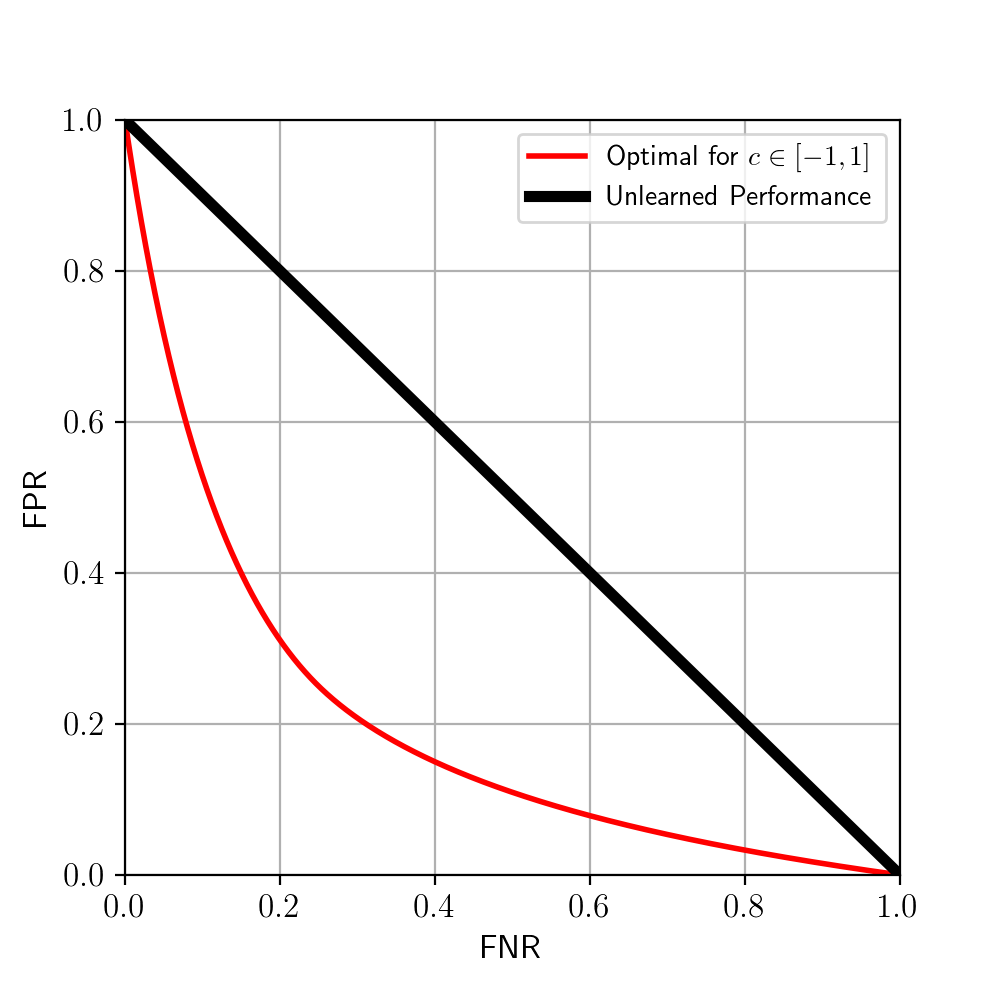}
\vspace{-10pt}
\caption{Continuum of optimal predictors (red) for varying class population ratios (successes and failures). The failure predictor can perform badly (i.e., on the bottom part of the curve) even with infinite data.}
\label{Fig:ToyROC}
\vspace{-10pt}
\end{figure}

Consider a single-step setting where the true state $x$ is the sum of the observation $o$ and an unobserved noise $\epsilon$ associated with an environment $E \sim \mathcal{D}_\E$. The observation and noise are drawn uniformly on the interval $[-1, 1]$ (i.e., $o, \epsilon \sim \mathcal{U}(-1, 1)$). Define the true success-failure label $y_i$ for a policy $\pi$ and environments $E_i$ as:
\begin{equation}\label{Eq:ToyLabel}
    y_i = \mathbbm{1}\big[[x_i \triangleq o_i + \epsilon_i] \geq c\big],
\end{equation}
where $c \in [-2, 2]$ is a known failure threshold.
This threshold allows us to vary the expected fraction of $1$-labels (failures) from the underlying distribution $\D_\E$. If the predictor is aware of the uncertainty in observation and noise, then the optimal predictor in expectation under Eq.~\eqref{Eq:CostNominal} is,
\begin{equation*}
    \hat{y}^* = f^*(o) = \mathbbm{1}[o \geq c].
\end{equation*}

However, this implies that for the case $\lvert c \rvert \geq 1$, the optimal prediction is always the more-likely class, since $o \in [-1, 1]$. But even for $\lvert c \rvert = 1$, the less-likely class can still occur up to 12.5\% of the time (see Appendix~\ref{SubsecAppendix:Toy} for calculations). Then the predictor always predicts success and misses all failures; the misclassification error is $0.125$. This shows the misclassification error biases the failure predictor towards the more common class. If that class is benign, the predictor under-detects failures.

This simple problem motivates us to instead minimize the \emph{conditional probability of misclassification} of the predictor, which is the \emph{weighted} sum of the FNR and FPR. Fig.~\ref{Fig:ToyROC} shows conditional failure probabilities of the optimal predictors (red curve) as the parameter $c$ varies in $[-1, 1]$. Although varying $c$ only changes the true proportion of classes 0 and 1 among training samples, the conditional performance of the predictor can degrade considerably depending on $c$. Even with infinite data, the optimal predictor will \emph{never} predict a failure when at least 87.5\% of training samples are successes. For analytical details regarding the policy and predictor, see Appendix~\ref{SubsecAppendix:Toy}.

%% file: RSS/Sections/Approach_Folder/ClassCondPAC.tex
\subsection{Bound on Class-Conditional Misclassification Error}
\label{subsecApp:Class}
The example above shows that minimizing the total misclassification error can fail to perform well when the relative importance of each class does not match its prevalence in the data (i.e., failures are important but rare). 
Our direct approach to generalize the method therefore utilizes a weighted misclassification error. Interpreting the total misclassification error as the class-weighted sum of conditional error probabilities, the following shows the desired generalization to an arbitrary weight $\lambda \in [0, 1]$:
\begin{equation} \label{Eq:Nominal2Class}
    \begin{split}
        p_\text{error} & = p_{0 \cap 1} + p_{1 \cap 0} \\
        & = p_{0 \rvert 1} p_1 + p_{1 \rvert 0}p_0 \\
        & = p_{0 \rvert 1} (1-\lambda^*) + p_{1 \rvert 0}(\lambda^*) \\
        \rightarrow \text{ generalize to } & = p_{0 \rvert 1} (1-\lambda) + p_{1 \rvert 0}(\lambda),
    \end{split}
\end{equation}

where $p_0$ and $p_1$ are true proportions of successes and failures in environments from $\D_\E$. Since the true FNR and FPR, $p_{0 \rvert 1}$ and $p_{1 \rvert 0}$, are unknown, we use the empirical FNR and FPR, $\hat{p}_{0 \rvert 1}$ and $\hat{p}_{1 \rvert 0}$, present in the dataset $S$ instead and propose the following error function as the empirical mean over environments in $S$:
\begin{align}\label{Eq:ClassCost}
    \hat{C}_S(r_f(E, \pi), S) 
    \triangleq (1-\lambda)\hat{p}_{0 \rvert 1} + \lambda \hat{p}_{1 \rvert 0}.
\end{align}
However, these empirical proportions are random variables over the draw of $S$, and thus $\hat{C}_S(r_f(E, \pi), S)$ will not admit a PAC-Bayes bound. Instead, we need an alternate error function that holds \emph{with high probability} over dataset $S$ and is a provably tight over-approximation of \eqref{Eq:ClassCost}. Defining $\underline{p}_i$ to be high-probability lower bounds for $p_i$ (found via Bernstein's Inequality), the necessary error function is \eqref{Eq:Theorem3Proof}. 
\begin{equation}
\label{Eq:Theorem3Proof}
    \begin{split}
        \tilde{C}(r_f(E, \pi)) & = \frac{\lambda}{C_\lambda \underline{p}_0} \mathbbm{1}[(\max_{t<T_\text{fail}} \hat{y}_t) =1 \cap y=0] \\
        & + \frac{1-\lambda}{C_\lambda \underline{p}_1} \mathbbm{1}[(\max_{t<T_\text{fail}} \hat{y}_t) =0 \cap y=1], \\
        \text{where } C_\lambda & \triangleq  \frac{\lambda}{\underline{p}_0} + 
        \frac{1-\lambda}{\underline{p}_1}.
    \end{split}
\end{equation}

The cost function is bounded in $[0, 1]$ and does not rely on quantities that are random variables over the draw of $S$. Now, we can present the class-conditional bound. 
\begin{theorem}
[Class-Conditioned PAC-Bayes Bound]
\label{ClassTheorem}

\noindent For any $\delta \in (0, 1)$ and any prior distribution $\D_{\F,0}$ over failure predictors, the following inequality holds with probability at least $1-\delta$ over training samples $S \sim \D_\E^N$ for all posterior distributions $\D_\F$:
\begin{equation} \label{Eq:TheoremThree}
    \begin{split}
        \underset{E \sim \D_\E}{\EE}\underset{f \sim \D_\F}{\EE}  \big[\tilde{C}(r_f(E, \pi)) \big]  \leq \hat{C}_S(r_f(E, \pi), S) + R_\lambda, \\
        R_\lambda = \frac{5}{3}\sqrt{\frac{(1-\underline{p})\log{\frac{2}{\delta}}}{N\underline{p}}} + C_\lambda R(\D_\F, \D_{\F,0}, \delta). 
    \end{split}
\end{equation}
A key result is that the regularizer $R_\lambda$ remains $\tilde{O}(\frac{1}{\sqrt{N}})$. 
\end{theorem}
\begin{proof}
A detailed proof is deferred to Appendix \ref{SubsecAppendix:CCB}. In summary, we utilize a union bound to ensure that all necessary events hold with high probability, and then combine the new error functions into the PAC-Bayes framework to obtain the new `regularizer' term $R_\lambda$. 
\end{proof}

\textbf{Remark.} This framework achieves two principal benefits. First, the bounds meaningfully hold for settings where failures and successes are not balanced in the sample dataset. Additionally, for a given confidence level, only a fixed number of instances of the less-common class is necessary to achieve a bound. This reflects the intuition that the fundamental information constraint is the number of instances of the class that have been seen. 

\textbf{Remark.} Without the class-conditional bound, estimating the conditional accuracy of the detector (e.g., `what fraction of failures are being predicted?') can only be done loosely (dividing $p_{1 \cap 1}$ by $p_1$, the latter of which is assumed to be small). Further, without the class-conditional objective, the predictor itself will be optimized to predict the more common class; that is, the objective function will implicitly be biased against its intended application (finding failures).

\begin{figure}[h!]
\vspace{-5pt}
\centering
\includegraphics[width=0.42\textwidth]{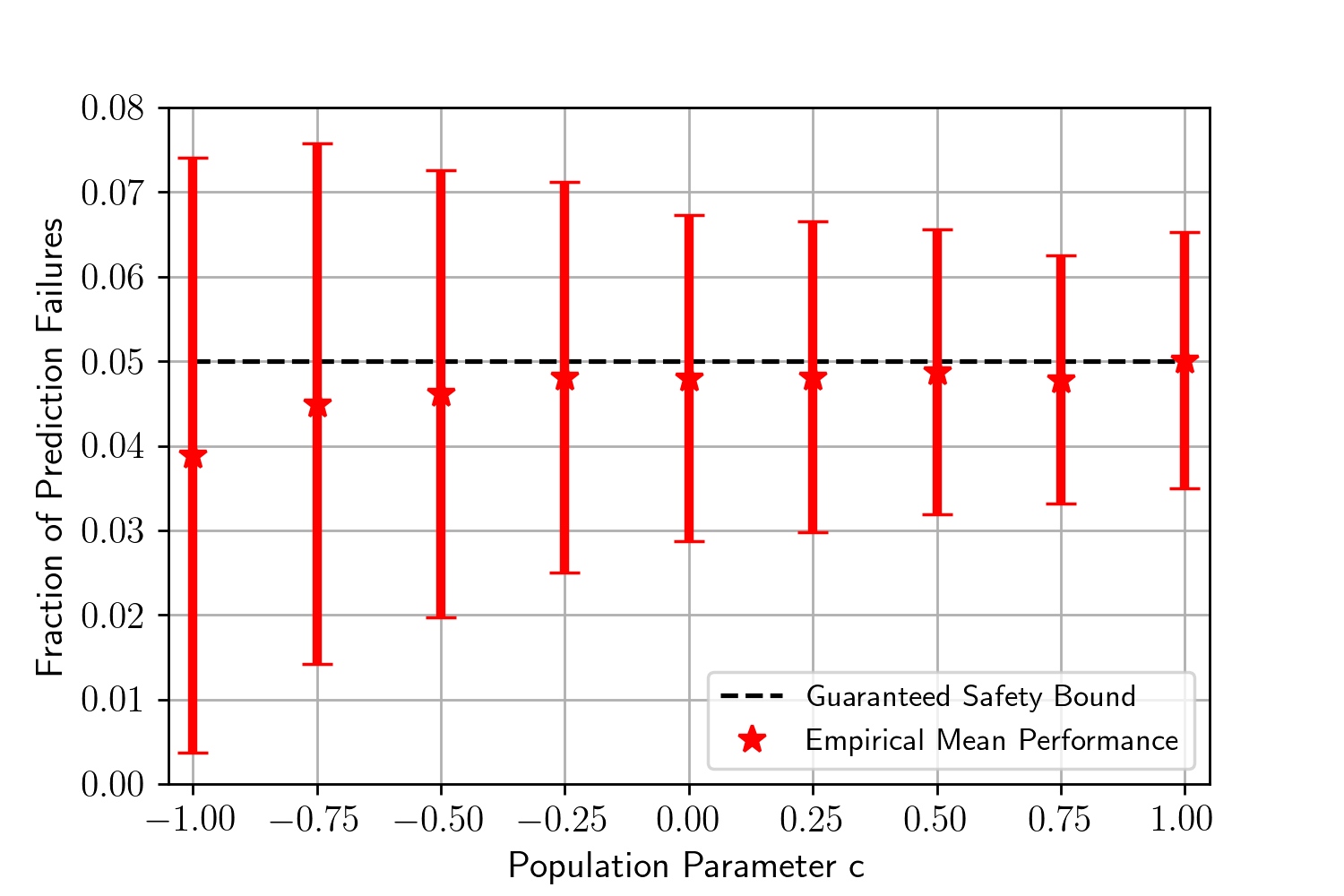}
\vspace{-5pt}
\caption{Results of conformal prediction on the toy problem in Sec. \ref{subsecApp:Toy}, which hold in \emph{expectation} over training samples $S$ and a test case (red dots). Empirical performance is better than the safety guarantee of $\epsilon^* = 0.05$ (black dashed line), as expected. However, the variance over predictors illustrated in the $1$-$\sigma$ error bars indicates that many predictors fail the guaranteed safety bound.}
\label{CP_Fig}
\vspace{-3mm}
\end{figure}

\textbf{Comparison to Conformal Prediction.}
Recently, techniques based on \emph{conformal prediction} have been proposed for learning failure predictors with error bounds \cite{luo_CP_safety_2021}. Here, we illustrate the fundamental differences between the guarantees provided by conformal learning and our PAC-Bayes approach. The conformal prediction framework allows for guarantees with better sample efficiency ($\tilde{O}(1/N)$, compared to $\tilde{O}(\sqrt{1/N})$ for PAC-Bayes bounds) that hold in expectation over both the test sample and training datasets. PAC-Bayes-style guarantees, on the other hand, hold in expectation over the test sample but \emph{with high probability} over training datasets. 

The expected performance of conformal prediction, shown in the red dots of Fig.~\ref{CP_Fig} meets the guaranteed bound (dashed black line) at $\tilde{O}(1/N)$ sample complexity. However, many of the synthesized failure predictors (the decision rules learned by conformal prediction using different training datasets) do not meet the desired safety levels of performance. The $1$-$\sigma$ error bars in Fig.~\ref{CP_Fig} represent variance over the predictors, indicating that many predictors themselves will violate the guaranteed safety bound. Thus, if the policy designer has access to a single training dataset to learn a failure predictor, conformal prediction does not guarantee that the expected error of the resulting policy will be below the desired threshold. This highlights a key difference between the approaches: the fraction of `bad' decision rules can be made arbitrarily small in the PAC-Bayes approach by reducing the parameter $\delta$; in conformal prediction, reducing the fraction of `bad' predictors is non-trivial.

%% file: RSS/Sections/Approach_Folder/Implementation.tex
\subsection{Algorithmic Implementation}
\label{sec:AlgorithmicApproach}
In the experimental setups, we parameterize failure predictors using neural networks and specify distributions over failure predictors using multivariate Gaussian distributions over the weights of the networks. We set the covariance matrix to be diagonal; thus, we can write the prior as $\D_{\F, 0}  = \mathcal{N}(\mu_0, \text{diag}(s_0))$, and posterior as $\D_\F = \mathcal{N}(\mu, \text{diag}(s))$. We train the posterior $\D_\F$ by optimizing the bounds on error rates provided by our theory. 
We present the training algorithm along with further implementation details in Appendix~\ref{subsec:expdetails}. Before training the posterior, we first train the prior $\D_{\F, 0}$ on held-out data to improve performance and resulting guarantees. After training is complete, we can use the posterior to compute PAC-Bayes bounds on error rates. The PAC-Bayes bound on the misclassification error can be computed directly with \eqref{eq:pacbayes_bound}. To compute PAC-Bayes bounds on the FNR and FPR, we can use $\lambda = 0$ and $\lambda = 1$ respectively in \eqref{Eq:TheoremThree}.

%% file: RSS/Sections/Exp_Results.tex
\section{Experimental Results} 
\label{sec:Experiments}

Through extensive simulation and hardware experiments, we aim to demonstrate strong guarantees on  (class-conditional) misclassification error of trained predictors. We also validate the guarantees by evaluating the predictors on test environments in both simulation and on hardware.

\input{RSS/Sections/Experiments/Drone}

\input{RSS/Sections/Experiments/Grasp}

%% file: RSS/Sections/Experiments/Drone.tex
\subsection{Obstacle Avoidance with a Drone}
\label{subsecER:Drone}

\textbf{Overview.} In this example, we train a failure predictor for a drone executing an obstacle-avoidance policy. The failure predictor uses depth images taken from the perspective of the drone. We train the failure predictor in simulation and apply it on a hardware platform with a Parrot Swing drone (Fig.~\ref{fig:anchor}); this is an agile quadrotor/fixed-wing hybrid drone that is capable of vertical takeoff/landing, and hovering in place. We consider a failure to be any instance in which the drone collides with an obstacle. If the failure predictor indicates that a crash will occur, we trigger an emergency policy that causes the drone to stop and land safely. In this section, we evaluate our bounds and failure prediction performance in simulation and also present extensive hardware expriments. 

\textbf{Policies and Environment Distributions.} The drone's obstacle avoidance policy consists of a deep neural network (DNN) classifier which uses a depth image to select one out of a set of open-loop motion primitives for the drone to execute. We use a Vicon motion capture system to locate all obstacles in the environment as well as the drone's position to generate artificial depth images from the perspective of the drone. It is important to note that neither the policy nor the failure predictor has access to any other information about the environment besides the first-person depth images. 

In order to match our simulation setup to the hardware system, we generate a set of motion primitives for the drone by applying sequences of open-loop control actions on the hardware platform. We then record the resulting trajectories using a Vicon motion capture system and use these trajectories for our simulation experiments (in order to emulate the hardware implementation). Thus these motion primitives capture noise in the hardware dynamics and help to bridge the sim-to-real gap.

We use two distinct policies and environment distributions in order to demonstrate our approach. In the first (`standard') setting, obstacles are placed (uniformly) randomly in the environment; the drone's policy then selects from a set of motion primitives which cause the drone to travel in mostly straight paths (Fig.~\ref{fig:trajectories} left). In the second (`occluded') setting, the environment is generated in two stages. First, a number of obstacles are (uniformly) randomly placed. Second, we generate additional obstacles that are placed exclusively in locations that are occluded by obstacles generated in the first stage. The drone's policy then uses the depth image captured from its starting location to select a motion primitive from the set shown in Fig.~\ref{fig:trajectories} (right). These primitives cause the drone to travel in curved paths (which may intersect initially-occluded obstacles). Since obstacles in this setting are specifically placed in portions of the environment that are occluded, this setting is adversarial in nature and leads to a large number of failures. The policy for the standard setting results in a failure rate of $25.3\%$, while the policy for the occluded setting results in a failure rate of $51.4\%$. Our approach allows for the use of a black-box policy which is not necessarily deployed in settings where it was trained nor where it would perform well. We aim to examine the improvement in safety of the policy with the addition of the failure predictor; thus, we test in settings that are challenging and even adversarial to the policy.

\begin{figure}[t]
\vspace{-5pt}
\begin{center}
    \includegraphics[width=0.45\textwidth]{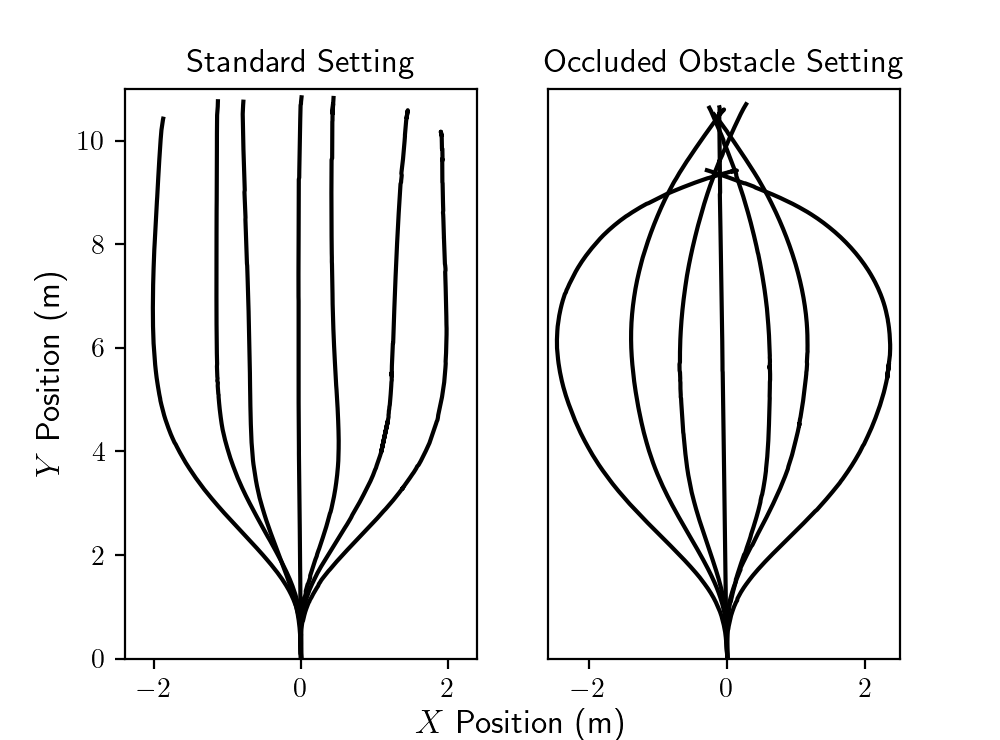}
\end{center}
\vspace{-8pt}
\caption{(Left) Representative motion primitives for the policy used in the standard setting in the navigation task. (Right) Representative motion primitives from the policy used in the occluded obstacle settings. }
\label{fig:trajectories}
\vspace{-13pt}
\end{figure}

\begin{figure}
\centering
\includegraphics[width=0.40\textwidth]{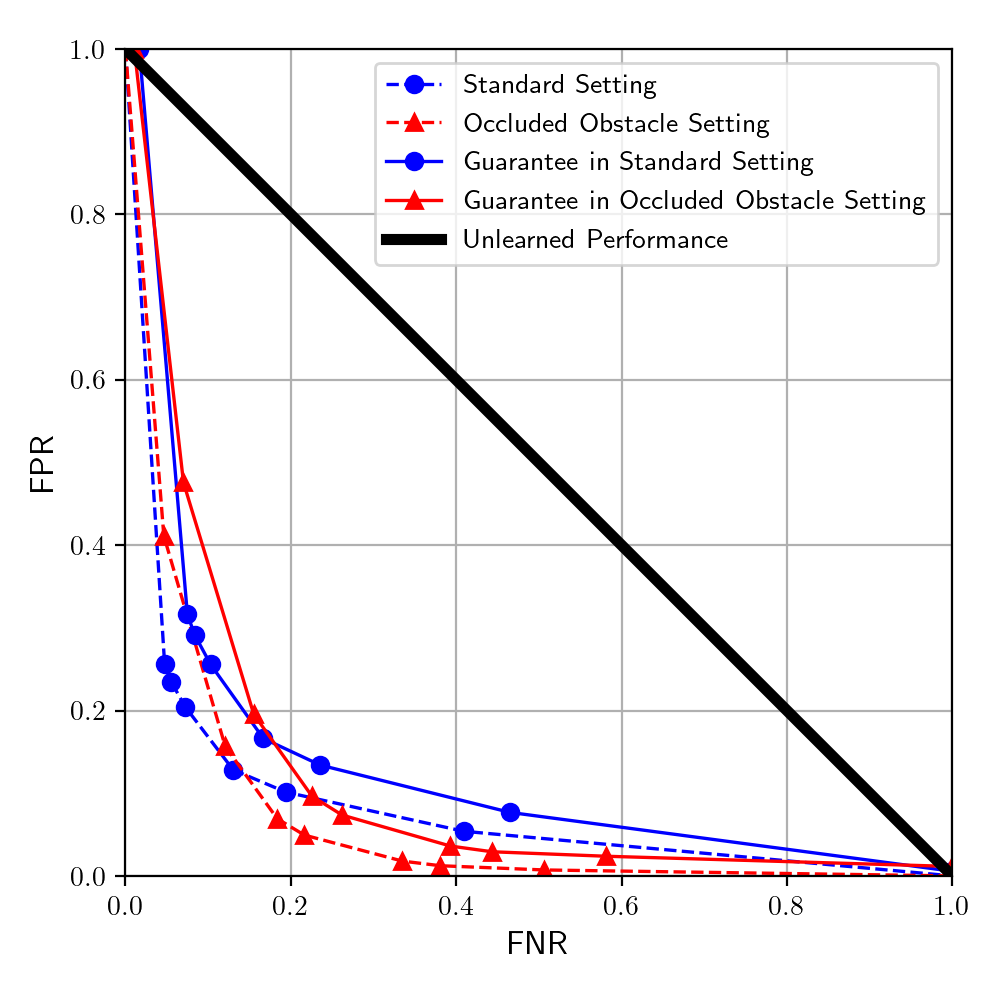}
\vspace{-10pt}
\caption{The dashed lines show the resulting failure predictor's FNR and FPR when the importance of a false negative is varied. The solid lines and markers show the associated PAC-Bayes guarantee on the FNR and FPR using the class-condition guarantees introduced in Sec.~\ref{subsecApp:Class}.}
\label{fig:navigation_fprvfnr}
\vspace{-15pt}
\end{figure}

\textbf{Failure Predictor and Training.} The failure predictor consists of a DNN which takes as input depth images from the perspective of the drone. The network receives a new image at a frequency of 20Hz and stacks the four most recent images as input to predict failure or no failure at some point in the future (i.e. 2-classes), with a Softmax layer at the end. If at any time step, the failure class is largest, we say it is a prediction of failure; otherwise, it is a prediction of success. We use 10,000 training environments to train the prior distribution $\D_{\F_0}$ over failure predictors and an additional 10,000 training environments to optimize the PAC-Bayes upper bound in \eqref{eq:pacbayes_bound}. We use Algorithm~\ref{alg:backprop} to obtain a posterior distribution $\D_\F$ and the associated PAC-Bayes generalization guarantee. We note that the physical drone takes approximately $0.5s$ to slow to a stop; thus, the failure predictor needs to predict collisions 10 time-steps ahead (since the frequency of the input is 20Hz).

After training, we compute the bound on the misclassification error with \eqref{eq:pacbayes_bound}, in addition to the bounds on the FNR and FPR with \eqref{Eq:TheoremThree}. At test time, we sample a failure predictor $f \sim \D_\F$ and fix it for a particular test environment. 

\textbf{Simulation Results.} We first verify our failure predictors in simulation before testing them on the hardware setup. As shown in Sec.~\ref{subsecApp:Toy}, there is an intrinsic trade-off between the false-positive and false-negative rates in general. To validate our theoretical framework, we plot the bounds on the FNR and FPR given by \eqref{Eq:TheoremThree} (solid lines in Fig.~\ref{fig:navigation_fprvfnr}). We compare these bounds with the empirical FNR and FPR (estimated with 20000 held out environments) of the predictor with varied false-negative importance during training; these are plotted using dashed lines in Fig.~\ref{fig:navigation_fprvfnr}. 
We obtain strong bounds on the failure predictors' errors, with closely-matching empirical performance along all points along the curve. 

\textbf{Hardware Results.} 
We select a single predictor for each of the settings (standard and occluded) to use in the hardware experiments. We choose the failure predictors with the tightest guaranteed total error rate (i.e. tightest upper bound from \eqref{eq:pacbayes_bound}). When the failure predictor stops the rollout due to a prediction of failure, we re-run the trial without the failure predictor to determine the true label. \ifarxiv 
See Fig.~\ref{fig:anchor} for an example of the obstacle placement for the occluded obstacle setting. Additionally, a video with representative trials from both settings is available at \href{https://youtu.be/z4UwQzTjhqo}{\texttt{https://youtu.be/z4UwQzTjhqo}}.
\fi
We run 15 trials in each of the settings and show the misclassification error of the failure predictors along with the guarantee on the misclassification in Table \ref{tab:navigation-result}. In the standard setting, $p_{1 \cap 1}=\frac{2}{15},\ p_{0 \cap 0}=\frac{12}{15},\  p_{1 \cap 0}=\frac{1}{15},\ p_{0 \cap 1}=0$, and in the occluded obstacle setting $p_{1 \cap 1}=\frac{9}{15}, \ p_{0 \cap 0}=\frac{4}{15}, \ p_{1 \cap 0}=\frac{2}{15}, \ p_{0 \cap 1}=0$. The misclassification bounds are validated by results from both the standard and occluded obstacle settings. Additionally, in all trials where the failure predictor was used, the drone never crashed due to a missed failure (i.e., there were no false negatives). These experiments demonstrate the power of the failure predictor to maintain the robot's safety in novel and even adversarial settings. 

\begin{table}[t]
\footnotesize
    \centering
    \caption{Results for Failure Prediction on Navigation Task}
    \begin{tabular}{ccc}
        \toprule
        Setting & Standard & Occluded Obstacle \\
        \midrule
        True Expected Failure (Sim)             & 0.253 & 0.514 \\ \midrule
        Misclassification Bound                  & 0.128 & 0.154  \\ 
        True Expected Misclassification (Sim)    & 0.101 & 0.125 \\
        True Expected Misclassification (Real)  & 0.067 & 0.133 \\
        \bottomrule
    \end{tabular}
    \label{tab:navigation-result}
    \vspace{-10pt}
\end{table}

%% file: RSS/Sections/Experiments/Grasp.tex
\subsection{Grasping Mugs with a Robot Arm}
\label{subsecER:Grasp}

\textbf{Overview.} In this example, we consider a robot arm performing the task of grasping a mug (Fig.~\ref{fig:grasp-envs}). There are three mugs on the table and the arm needs to lift one of them off the table. We use the PyBullet simulator \cite{coumans2021} for training the policy and failure predictor, and then test these on a hardware setup with a Franka Panda 7-DOF arm equipped with a wrist-mounted Intel RealSense D435 camera (Fig.~\ref{fig:grasp-envs}). Failure is defined as either (1) the arm failing to lift the mug from the table or (2) the gripper contacting the mug with excessive force before grasping (e.g., if the bottom of the gripper finger hits the rim of the mug). In simulation, we check the force readings at the end-effector joint of the arm, and set the threshold to be $20$N for (2). We use the same threshold on the physical arm. 

\textbf{Policy and Environment Distributions.} 
We consider a multi-view grasping setting: with a fixed horizon of five steps, the robot starts at a fixed pose and gradually moves towards the mug. The gripper always closes at the last step to grasp the mug. The policy takes a $100 \times 150$ pixels RGB-D image from the camera and outputs the desired pose relative to the current pose, $(\Delta_x, \Delta_y, \Delta_z, \Delta_\psi)$, in 3D translation and yaw orientation. The relative pose is normalized between $[-0.02,0.02]$cm for $\Delta_x$ and $\Delta_y$, $[-0.05,-0.03]$cm for $\Delta_z$, and $[-15,15]^\circ$ for $\Delta_\psi$. To simulate more realistic camera images, we add pixel-wise Gaussian noise to the depth image, and randomly change the color values for $5\%$ of the pixels in the RGB image. In order to evaluate the failure predictor on policies with different task success rates, we choose three different policies saved at different epochs during training. Shown in Table~\ref{tab:grasp-result}, the policies (`High', `Medium', and `Low' settings) achieve $11.8\%$, $22.1\%$, and $46.6\%$ failure rates respectively on test environments in simulation.

In order to create different environments for the robot, we obtained 50 mugs of diverse geometries from the ShapeNet dataset \cite{wu20153d}. We then randomly scale them to different diameters between $[8, 12]$cm. For each environment, we sample three different mugs and their initial locations on the table (while ensuring that the mugs do not intersect). We randomly assign a RGB value as the color of the mug.

\textbf{Failure Predictor and Training.} Similar to the drone example, the predictor is parameterized with a DNN that takes the RGB-D image from the camera and outputs the probability of success and failure with a (two-class) Softmax layer at the end. If the predicted probability of failure is higher than $0.5$, the predictor outputs $\hat{y}=1$. Since the gripper does not reach the mug in the first two out of five steps of the rollout, we only train the failure predictor using data from the last three steps of each rollout. We use 5000 training environments to train the prior distribution $\D_{\F,0}$ over failure predictors and an additional 5000 training environments to optimize the PAC-Bayes upper bound and provide a posterior $\D_\F$. 

\begin{figure}
\centering
\includegraphics[width=0.40\textwidth]{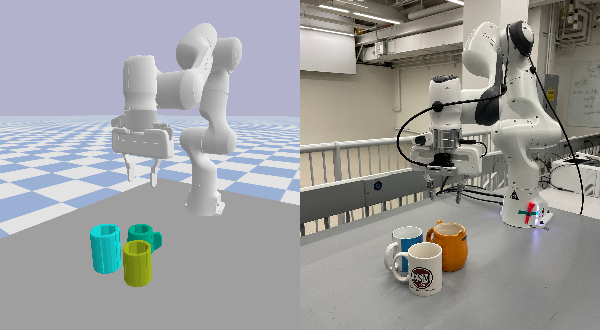}
\caption{(a) Simulation environment in PyBullet simulator (virtual wrist-mounted camera not shown). (b) Real environment with an arm and a wrist-mounted camera for grasping task.}
\label{fig:grasp-envs}
\vspace{-10pt}
\end{figure}

\begin{figure}
\centering
\includegraphics[width=0.45\textwidth]{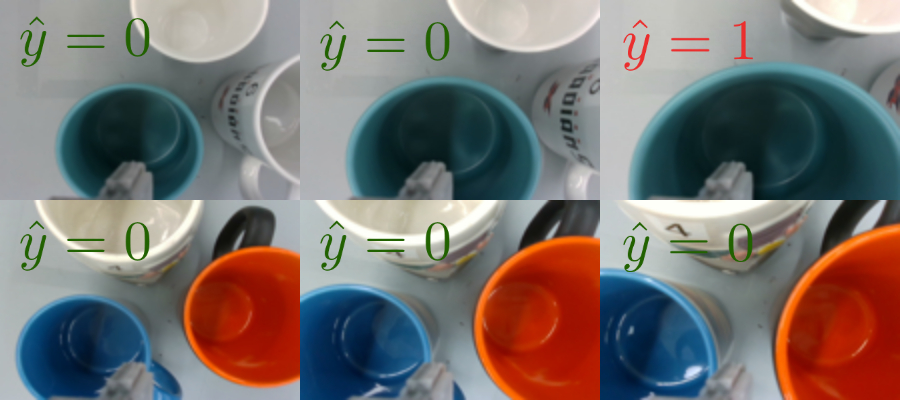}
\caption{Camera observation (depth not shown) at the last three steps in a trial in the grasping task on the hardware platform. (Top) A true-positive trial, where the predictor outputs failure at the last step, and then the right finger of the gripper hits the mug rim (not shown); (bottom) a true-negative trial.}
\label{fig:grasp-trajs}
\vspace{-10pt}
\end{figure}

\textbf{Simulation Results.} 
Generalization bounds and test prediction failure rates of the grasping example are shown in Table~\ref{tab:grasp-result}. Across all three settings, we achieve tight guarantees on failure prediction compared to the true expected failure rate of the policies. When testing on additional 5000 environments in simulation (`Sim'), the expected prediction rate validates the bound in all settings. We also plot different levels of bounds on FNR and FPR by varying the importance of false negatives relative to that of false positive in the `Medium' setting (see Fig.~\ref{fig:grasp-bound} in Appendix \ref{subsec:expdetails}). Similar to Fig.~\ref{fig:navigation_fprvfnr} in the drone example, we achieve strong bounds on conditional misclassification errors and the empirical FNR and FPR match the bounds.

\begin{table}[!ht]
\footnotesize
    \centering
    \caption{Results of Failure Prediction in Grasping Task}
    \begin{tabular}{ccccc}
        \toprule
        Setting && High & Medium & Low \\
        \midrule
        True Expected Failure (Sim) && 0.124 & 0.223 & 0.485 \\
        \midrule
        Misclassification Bound && 0.087 & 0.130 & 0.142 \\
        True Expected Misclassification (Sim) && 0.058 & 0.101 & 0.118 \\
        True Expected Misclassification (Real) && 0.067 & 0.133 & 0.133 \\
        \bottomrule
    \end{tabular}
    \label{tab:grasp-result}
    \vspace{-10pt}
\end{table}

\textbf{Hardware Results.} We also test the failure predictor on the hardware setup (Fig.~\ref{fig:grasp-envs}). We collect 18 real mugs of diverse geometries and visuals and split them into 6 sets. For each set, we perform 5 experiments with random initial mug poses and randomly sampled failure predictor from the posterior distribution. In order to check if the prediction is correct, we continue the trial even when the predictor detects failure. Fig.~\ref{fig:grasp-trajs} shows the RGB images from the camera and failure predictions at the last three steps of two trials. \ifarxiv A video with representative trials is available at \href{https://youtu.be/z4UwQzTjhqo}{\texttt{https://youtu.be/z4UwQzTjhqo}}. \fi

The results of 30 trials for each setting are shown in Table~\ref{tab:grasp-result} (Real). In the `High' setting, $p_{1 \cap 1}=\frac{4}{30},\ p_{0 \cap 0}=\frac{24}{30},\  p_{1 \cap 0}=\frac{2}{30},\ p_{0 \cap 1}=0$; in the `Medium' setting $p_{1 \cap 1}=\frac{6}{30}, \ p_{0 \cap 0}=\frac{20}{30}, \ p_{1 \cap 0}=\frac{3}{30}, \ p_{0 \cap 1}=\frac{1}{30}$; and in the `Low' setting $p_{1 \cap 1}=\frac{11}{30}, \ p_{0 \cap 0}=\frac{15}{30}, \ p_{1 \cap 0}=\frac{3}{30}, \ p_{0 \cap 1}=\frac{1}{30}$. In both `High' and `Low' settings, the empirical results validate the corresponding misclassification bounds. In the `Medium' setting, the empirical rate is slightly higher than the bound. This could be due to bias from the limited number of real experiment trials and disparities between camera observations in simulation and in real. Nonetheless, the hardware results overall demonstrates the tightness of the theoretical guarantees on failure prediction.

%% file: RSS/Sections/Conclusion.tex
\section{Conclusion} 
\label{sec:conclusion}

In this work we propose an approach for training a failure predictor with rigorous guarantees on error rates by leveraging PAC-Bayes generalization theory. We develop novel class-conditional bounds that allow us to trade-off the relative rates of false-positive and false-negative errors. We present an algorithmic approach that parameterizes the predictor using a neural network and minimizes the theoretical error bounds. Through extensive experiments in vision-based drone navigation and robot grasping tasks in both simulation and on real hardware, we demonstrate strong guarantees on failure prediction rates.

{\bf Future work.} Our approach assumes that environments that the failure predictor will be deployed in are drawn from the same underlying (but unknown) distribution from which training environments are generated. Extending our approach to perform failure prediction in \emph{out-of-distribution} environments (e.g., using \emph{distributionally robust} bounds) would be of significant interest. In addition, the results we present in this paper are for relatively short-horizon tasks; we are interested to test our approach on longer-horizon problems (e.g., using transformers or long short-term memory models).
We also note that in the settings we test, the black-box policies have as high as 87.5\% success rate. However, a limitation of our approach, and of statistical approaches in general, is developing meaningful guarantees in settings where the policy has a significantly higher (e.g., 99.9\%) success rate. Scaling to such settings where failures are extremely rare may require a large number of training samples without further assumptions about the policy or training data; this is a good direction for future work. Finally, the use of more sophisticated network architectures (e.g., those based on attention models) could have an impact on the quality of the bounds and empirical performance. In particular, attention models could learn to focus on particular portions of the robot's observations that are most predictive of failures.

%% file: RSS/Sections/Appendices/Class_Conditional_Bound.tex
\section{Class-Conditional Bound Proof} \label{SubsecAppendix:CCB}

\begin{proof}[\textbf{Proof of Theorem 3}]

\begin{lemma}\label{BernsteinLemma}
(Bernstein Inequality for Bernoulli Variables)

For any $\delta \in (0, 1)$, $p_0 \in (0, 1)$ and $p_1 = 1-p_0$, with probability greater than $1-\delta$ over samples $S$ of size $N$ of i.i.d. Bernoulli random variables, the following inequalities hold for $\underline{p}_0, \underline{p}_1$ and ratios $\underline{K}_i$: 

\begin{equation}
    \begin{split}
        \underline{K}_i & = \frac{3}{5}\Big[\sqrt{\frac{N\underline{p}_i}{2\log{\frac{2}{\delta}}}}\Big], \\
        p_i & \geq \hat{p}_i - \frac{1}{\underline{K}_i + 1}\hat{p}_i.
    \end{split}
\end{equation}
\end{lemma}
\begin{proof}
For the empirical mean of Bernoulli random variables, the Bernstein inequality allows us to show directly that 
\begin{equation*}
    \mathbb{P}\Big[\hat{p}_i - p_i \geq \frac{10}{3}\sqrt{\frac{\sigma^2 \log{\frac{2}{\delta}}}{N}}\Big] \leq \frac{\delta}{2}.
\end{equation*}
Defining $K_{\delta N} = \frac{100 \log{\frac{2}{\delta}}}{9}$, the above term inside the brackets simplifies to \begin{equation*}
    p_i^2(1 + K_{\delta N}) - (2\hat{p}_i + K_{\delta N})p_i + \hat{p}_i^2 \leq 0.
\end{equation*}
This is a scalar quadratic equation in $p_i$; taking the minimal solution gives $\underline{p}_i$. Now, analyze the multiplicative bound 
\begin{equation*}
    \frac{10}{3}\sqrt{\frac{\sigma^2 \log{\frac{2}{\delta}}}{N}} \leq \frac{p_i}{K_i}.
\end{equation*}
This bound implies 
\begin{equation*}
    \begin{split}
        K_i & \leq \frac{3}{5}\sqrt{\frac{N p_i}{2\log{\frac{2}{\delta}}}}, \\
        N & \geq \frac{50 K_i^2 \log{\frac{2}{\delta}}}{9p_i}.
    \end{split}
\end{equation*}
To be practically feasible, we must be able to bound $p_i$ away from $0$. This is equivalent to achieving $K_i > 1$. Of course, observing a few instances (or even a single instance) of class $i$ is sufficient to ensure this condition is met for a Bernoulli variable. Assuming this holds, we have that 
\begin{equation*}
    \begin{split}
        \hat{p}_i - p_i & \leq \frac{p_i}{K_i}, \\
        \implies \hat{p}_i & \leq p_i (1 + \frac{1}{K_i}), \\
        \implies p_i & \geq \frac{K_i}{K_i+1}\hat{p}_i, \\
        & = \hat{p}_i - \frac{1}{K_i+1}\hat{p}_i \triangleq \underline{p}_i.
    \end{split}
\end{equation*}
\end{proof}
Using the error functions, we can rearrange to show that the structure is maintained. Specifically, 
\begin{equation*}
    \begin{split}
        \hat{C}(r_f(E, \pi),S) & = \sum_{i=0}^1 \frac{\lambda_i}{\hat{p}_i} \mathbbm{1}[f(E) \neq y], \\
        C_\lambda & \triangleq \frac{\lambda_0}{\underline{p}_0} + \frac{\lambda_1}{\underline{p}_1}, \\
        \tilde{C}(r_f(E, \pi)) & = \frac{1}{C_\lambda} \sum_{i=0}^1 \frac{\lambda_i}{\underline{p}_i} \mathbbm{1}[f(E)\neq y], \\
        & = \sum_{i=0}^1 \lambda'_i \mathbbm{1}[f(E)\neq y]. \\
    \end{split}
\end{equation*}
Because $\tilde{C}$ is a valid error function, we can apply the PAC-Bayes framework. Defining for notational shorthand 
\begin{equation*}
    R_{PAC} \triangleq C_\lambda \sqrt{\frac{\KL(\D_\F \rVert \D_{\F,0}) + \log(\frac{2\sqrt{N}}{\delta})}{2N}},
\end{equation*}
we see that under standard PAC-Bayes, we can achieve the following bound: 
\begin{equation*}
    \underset{E \sim \D_\E}{\EE}\underset{f \sim \D_\F}{\EE}  \big[\tilde{C}(r_f(E, \pi)) \big] \leq \frac{1}{N}\sum_{E \in S}^N \tilde{C}(r_f(E, \pi))+ R_{PAC}.
\end{equation*}
This implies then that we can bound $\hat{C}$ as follows: 
\begin{equation*}
    \begin{split}
        \underset{E \sim \D_\E}{\EE}\underset{f \sim \D_\F}{\EE} & \big[\hat{C}(r_f(E, \pi),S) \big] \leq \underset{E \sim \D_\E}{\EE}\underset{f \sim \D_\F}{\EE} \big[C_\lambda \tilde{C}(r_f(E, \pi))\big], \\
        & \leq C_\lambda \frac{1}{N} \sum_{i=1}^N \tilde{C}(r_f(E, \pi)) + R_{PAC}, \\
        & \leq \frac{1}{N} \sum_{i=1}^N \hat{C}(r_f(E, \pi),S)(1 + \frac{1}{\underline{K}_{min}}) + R_{PAC}.
    \end{split}
\end{equation*}
The first and last inequalities follow from Lemma \ref{BernsteinLemma} and the fact that $\underline{p}_i < \hat{p}_i$. Specifically, the latter fact implies that $\hat{C} \leq C_\lambda \tilde{C}(\cdot)$ (in words: $C_\lambda \tilde{C}$ is an over-approximation), while the first fact implies that the over-approximation ratio is less than $1 + 1/K$ (i.e. is relatively tight). Defining $\underline{p} = \min_{i} \underline{p}_i$ and $\underline{K} = \min_{i} \underline{K}_i$, and noting that the empirical mean error must be less than $1/2$ by definition of $\lambda$ (that is, there exists a policy that achieves $\min\{\lambda, 1-\lambda\}$ cost by simply always guessing the appropriate class), we have that (conservatively): 
\begin{equation*}
    \begin{split}
        \underset{E \sim \D_\E}{\EE} & \underset{f \sim \D_\F}{\EE} \big[\hat{C}(r_f(E, \pi),S) \big] \leq \underset{E \sim \D_\E}{\EE}\underset{f \sim \D_\F}{\EE} \big[C_\lambda \tilde{C}(r_f(E, \pi)) \big], \\
        & \leq \frac{1}{N} \sum_{i=1}^N \hat{C}(r_f(E, \pi),S)(1 + \frac{1}{\underline{K}}) + R_{PAC}, \\
        & \leq \frac{1}{N} \sum_{i=1}^N \hat{C}(r_f(E, \pi),S) + \frac{1}{2\underline{K}} + R_{PAC}, \\
        & = \frac{1}{N} \sum_{i=1}^N \hat{C}(r_f(E, \pi),S) + \frac{5}{3}\sqrt{\frac{(1-\underline{p})\log{\frac{2}{\delta}}}{N\underline{p}}} + R_{PAC}.
    \end{split}
\end{equation*}
This matches the result of \eqref{Eq:TheoremThree} stated in Theorem \ref{ClassTheorem}, and completes the proof. 
\end{proof}

%% file: RSS/Sections/Appendices/ToyExample.tex
\section{Toy Example Calculations}
\label{SubsecAppendix:Toy}

Let $w_i = x_i + z_i$. Then for each training sample, the true class label is defined by $y_i = \mathbbm{1}[w_i \geq c]$. The optimal classifier in expectation is $\hat{y}^*_i = f^*(x_i) = \mathbbm{1}[x_i \geq c]$. It can be shown that the probability density for $W = X + Z$ is
\begin{equation}
    p(w) = 0.5 - 0.25\lvert w \rvert \text{, } w \in [-2, 2].
\end{equation}
Further, the probability of a new example being in classes 0 and 1 as a function of the cutoff $c \in [-2, 0]$ is 
\begin{equation}
\begin{split}
    p_0(c) & = \frac{(c+2)^2}{8}, \\
    p_1(c) & = 1 - p_0(c).
    \end{split}
\end{equation}
By symmetry, for $c \in [0, 2]$ the same probabilities hold but for the opposite quantities. We note that the optimal policy, based only on the observable $x_i$, is invariant for $\lvert c \rvert \geq 1$. As such, we will generally restrict our attention to $c \in [-1, 1]$.

Cap notation will again note the intersection of two events, which will be in the order $\{$predicted class, true class$\}$. I.e. $p_{1 \cap 0}$ denotes the probability that the predictor chooses class 1, but that the example is truly in class 0.

We now calculate $p_{1 \cap 0}$ and $p_{0 \cap 1}$ as a function of $c \in [-1, 0]$ (we will argue from symmetry that the results extend naturally to $c \in [0, 1]$, and argue by the optimal policy invariance for large-magnitude $c$ that the extension is trivial to $\lvert c \rvert > 1$). For clarity, because $X \sim \mathcal{U}[-1, 1]$, $d\mu_x = (1/2)dx$; this substitution will be very common throughout the following analysis: 
\begin{equation*}
    \begin{split}
        p_{1 \cap 0} & = p(\text{predict 1 } \cap \text{ truly 0}), \\
        & = p((x_i \geq c) \cap (w_i < c)), \\
        & = \int_c^{c+1} \frac{1 - (x-c)}{2} d\mu_x, \\
        & = \frac{1}{4} \int_c^{c+1} (1 + c - x) dx, \\
        & = \frac{1}{4} (1+c)x - \frac{x^2}{2} \rvert_c^{c+1}, \\
        & = \frac{1}{4} [(1+c) - (\frac{c^2 + 2c +1 -c^2}{2})], \\
        & = \frac{1}{8}.
    \end{split}
\end{equation*}
We have that for $c \in [-1, 0]$:
\begin{equation*}
    \begin{split}
        p_{1 \rvert 0}p_0 & = p_{1 \cap 0}, \\
        \implies p_{1 \rvert 0}(c) & = \frac{1/8}{(c+2)^2/8}, \\
        & = \frac{1}{(c+2)^2}, \\
        \implies p_{1 \rvert 0}(-1) & = 1, \\
        p_{1 \rvert 0}(0) & = \frac{1}{4}.
    \end{split}
\end{equation*} 
Similarly, again for $c \in [-1, 0]$:
\begin{equation*}
    \begin{split}
        p_{0 \cap 1} & = p(\text{predict 0 } \cap \text{ truly 1}), \\
        & = p((x_i < c) \cap (w_i > c)), \\
        & = \int_{-1}^c (1/2) - (1/2)(c-x) d\mu_x, \\
        & = \frac{1}{4} \int_{-1}^c (1 - c + x) dx, \\
        & = \frac{1}{4}\big[ (1-c)x + \frac{x^2}{2} \big] \rvert_{-1}^c, \\
        & = \frac{1}{4} \big[ (1-c)(c+1) + (c^2/2 - 1/2) \big], \\
        & = \frac{1}{8} \big[ 1 - c^2 \big].
    \end{split}
\end{equation*}
Which leads to the expression:
\begin{equation*}
    \begin{split}
        p_{0 \rvert 1}p_1 & = p_{0 \cap 1}, \\
        \implies p_{0 \rvert 1}(c) & = \frac{\frac{1}{8}(1-c^2)}{\frac{1}{8}(8-(c+2)^2)}, \\
        & = \frac{(1-c^2)}{(8-(c+2)^2)}, \\
        \implies p_{0 \rvert 1}(-1) & = \frac{0}{7} = 0, \\
        p_{0 \rvert 1}(0) & = \frac{1}{4}.
    \end{split}
\end{equation*} 

The total error is: 
\begin{equation*}
\begin{split}
    p_{err} & = p_{0 \cap 1} + p_{1 \cap 0}, \\
            & = \frac{1}{8} - \frac{c^2}{8} + \frac{1}{8}, \\
            & = \frac{1}{4} - \frac{1}{8}c^2.
    \end{split}
\end{equation*}
Finally, we calculate the derivative of the curve $(p_{1\rvert 0}(c), p_{0 \rvert 1}(c))$ parameterized by $c$:
\begin{equation*}
\begin{split}
    \frac{d (p_{1 \rvert 0})}{dc} & = \frac{-2}{(c+2)^3}, \\
     \frac{d (p_{0 \rvert 1})}{dc} & = \frac{(8-(c+2)^2)(-2c) - (1-c^2)(-2)(c+2)}{(8-(c+2)^2)^2}, \\
    & = \frac{4c^2 - 6c + 4}{(8-(c+2)^2)^2},
\end{split}
\end{equation*} \vspace{-.5em}
\begin{equation*}
\begin{split}
    \implies \frac{d (p_{0 \rvert 1})}{d (p_{1 \rvert 0})}(c) & = \frac{\frac{4c^2 - 6c + 4}{(8-(c+2)^2)^2}}{\frac{-2}{(c+2)^3}}, \\
    & = \frac{-(c+2)^3 (2c^2 - 3c + 2)}{(8-(c+2)^2)^2}, \\
    \implies \frac{d (p_{0 \rvert 1})}{d (p_{1 \rvert 0})}(-1) & = -\frac{1}{7}, \\
    \frac{d (p_{0 \rvert 1})}{d (p_{1 \rvert 0})}(0) & = -1.
\end{split}
\end{equation*}

%% file: RSS/Sections/Appendices/Experimental_details.tex
\section{Additional Implementation Details}
\label{subsec:expdetails}

\begin{algorithm}[t]
    \caption{Failure Predictor Training}
    \label{alg:backprop}
\begin{algorithmic}
    \State \textbf{Input}: Fixed prior distribution $\N_{\psi_0}$ over policies
    \State \textbf{Input}: Rollout function $r$, learning rate $\gamma$
    \State \textbf{Input}: Training dataset $S=\{E_i\}_{i=1}^N$
    \State \textbf{Output}: Optimized $\psi^*$
    \For{$i =\{1, 2, ..., N\}$}
    \State $\{(o_{i,j}, y_{i,j})\}_{j=1}^T \leftarrow r(E_i, \pi)$
    \EndFor
    \While{not converged}
    \State Sample $w \sim \N_\psi$
    \For{$i = \{1, 2, \dots, N\}$}
    \For{$j = \{1, 2, \dots, T\}$}
    \State $\hat{y}_{i,j} = f_w(\{o_{i,l}\}_{l=1}^j)$
    \EndFor
    \EndFor
    \State $B \leftarrow \frac{1}{N} \sum_{i=1}^N C_{s}(\{(\hat{y}_{i,j}, y_{i,j})\}_{j=1}^T)$ 
    \State \phantom{$B \leftarrow$} $+ \sqrt{\frac{\KL(\N_{\psi} \| \N_{\psi_0}) + \ln\frac{2\sqrt{N}}{\delta}}{2N}}$ 
    \State $\psi \leftarrow \psi - \gamma \nabla_{\psi} B$ 
    \EndWhile
\end{algorithmic}
\end{algorithm}

In Sec.~\ref{sec:Experiments}, the failure predictors are parameterized with deep neural networks (DNNs) and trained using stochastic gradient descent. We assume access to an additional rollout function $r: \E \times \Pi \rightarrow \mathcal{O}^{T} \times \mathcal{Y}^{T}$ which maps a given policy and environment to a series of observations that the failure predictor will use and true labels $y$ about whether the policy has failed or not, for $t\in \{1, \dots, T\}$. We let the prior and posterior distributions over failure predictors be multivariate Gaussian functions and write them as $\D_{\F, 0}  = \mathcal{N}(\mu_0, \text{diag}(s_0))$, and $\D_\F = \mathcal{N}(\mu, \text{diag}(s))$. Define $\psi = (\mu, \log s)$ and let $\mathcal{N}_\psi := \mathcal{N}(\mu, \text{diag}(s))$. We can then sample a single set of network weights $w$ from the distribution over weights $\mathcal{N}_\psi$ and let the associated failure predictor be $f_w$. The upper bound presented in \eqref{eq:pacbayes_bound} has a potentially intractable expectation taken over $f \sim \D_\F$. We use the same technique presented in \cite{Dziugiate17} to minimize the upper bound by sampling unbiased estimates of the expectation over $f \sim \D_\F$. Additionally, the error $C(r_f(E, \pi))$ we present as part of the upper bound in \eqref{eq:pacbayes_bound} is discontinuous in general, and may be difficult to minimize directly. As such, we use a surrogate error function $C_{s}$ in training. We let $C_{s}$ take in a series of failure predictions from $f \sim \D_\F$ and the true labels at each point in the rollout to assign an error for that series of predictions. We use a surrogate error similar to the cross entropy loss:
\begin{align}
\label{eq:surrogatecostfn}
    C_{s}(\{\hat{y}_j,y_j)\}_{j=1}^T) :=\frac{1}{T} \sum_{j=1}^T [& \omega (y_{\min(j+k, T)}\log (\hat{y}_j)) \\ 
    & + (1 - y_{\min(j+k, T)})\log (1 - \hat{y}_j)] \nonumber,
\end{align}
where $\omega$ represents a weight on the importance of a false negative as compared to a false positive and $k$ represents the number of time steps ahead that the failure predictor must detect a failure. This weight $\omega$ is analogous to $\lambda$ used in \eqref{Eq:TheoremThree} and provides a way to scale the importance of false negatives. We use this error function in training for our experiments and different values of $\omega$ and $k$. The resulting approach is presented in Algorithm~\ref{alg:backprop}. After training is complete, we compute the misclassification, FNR, and FPR guarantees using a sample convergence bound as in \cite{Dziugiate17}.


\vspace{20pt}
\begin{figure} 
\begin{center}
    \includegraphics[width=0.35\textwidth]{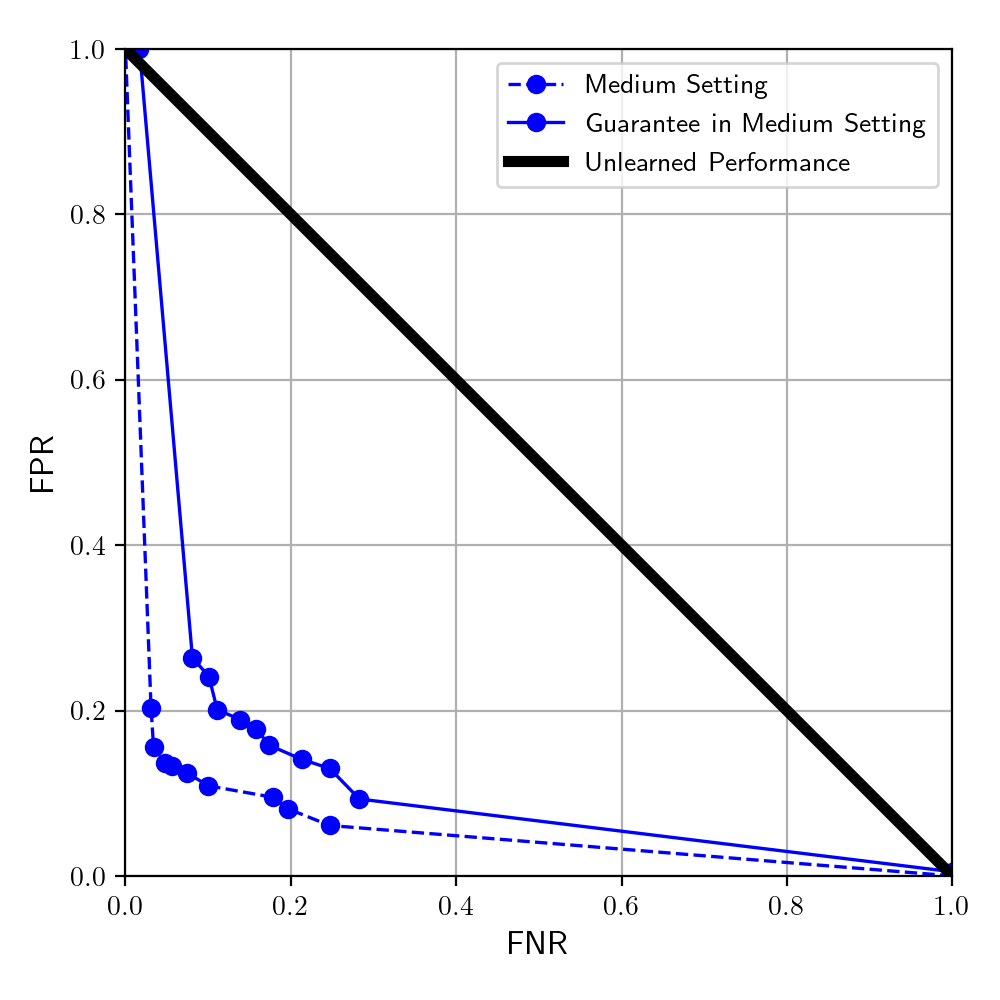}
\end{center}
\vspace{-15pt}
\caption{The dashed lines show the resulting failure predictor's FNR and FPR when the importance of a false negative is varied when training in the `Medium' setting in the grasping task. The solid lines and markers show the associated PAC-Bayes guarantee on the FNR and FPR using the class-condition guarantees introduction in Section \ref{subsecApp:Class}.}
\label{fig:grasp-bound}
\vspace{-15pt}
\end{figure}